\documentclass{article}
\usepackage[margin=1in]{geometry}
\usepackage[hyphens]{url}  % DO NOT CHANGE THIS
\usepackage{graphicx} % DO NOT CHANGE THIS
\usepackage{natbib}  % DO NOT CHANGE THIS AND DO NOT ADD ANY OPTIONS TO IT
\usepackage{algorithm}
\usepackage{algorithmic}
\usepackage{amsmath}
\usepackage{amssymb}
\usepackage{float}
\usepackage{amsthm} 
\usepackage{newfloat}
\usepackage{listings}
\usepackage{bm} 
\usepackage[font=small,labelfont=bf]{caption}
\floatstyle{ruled}
\newfloat{listing}{tb}{lst}{}
\floatname{listing}{Listing}

\title{Adaptive Stepsizing for Stochastic Gradient Langevin Dynamics in Bayesian Neural Networks}

\author{
  Rajit Rajpal \quad
  Benedict Leimkuhler \quad
  Yuanhao Jiang \\[6pt]
  School of Mathematics, University of Edinburgh \\[2pt]
  \texttt{s2592586@ed.ac.uk, B.Leimkuhler@ed.ac.uk, yuanhao.jiang@ed.ac.uk}
}

\usepackage{bibentry}

\newtheorem{theorem}{Theorem}
\newtheorem{lemma}{Lemma}

\date{}
\begin{document}
\maketitle

\begin{abstract}
\noindent Bayesian neural networks (BNNs) require scalable sampling algorithms to approximate posterior distributions over parameters.  Existing stochastic gradient Markov Chain Monte Carlo (SGMCMC) methods are highly sensitive to the choice of stepsize and adaptive variants such as pSGLD typically fail to sample the correct invariant measure without addition of a costly divergence correction term.  In this work, we build on the recently proposed `SamAdams' framework for timestep adaptation \cite{leimkuhler2025langevinsamplingalgorithminspired}, introducing an adaptive scheme: SA-SGLD, which employs time rescaling to modulate the stepsize according to a monitored quantity (typically the local gradient norm). SA-SGLD can automatically shrink stepsizes in regions of high curvature and expand them in flatter regions, improving both stability and mixing without introducing bias. We show that our method can achieve more accurate posterior sampling than SGLD on high-curvature 2D toy examples and in image classification with BNNs using sharp priors.
\end{abstract}

% \begin{links}
%     \link{Code}{https://github.com/rajit906/SAM-ADAMS-IN-BDL}
% \end{links}

\section{Introduction}
\noindent Bayesian Neural Networks (BNNs) provide a framework for quantifying uncertainty in deep learning models by placing a posterior distribution over the weights, $p(\bm{\theta}|\mathcal{D})$. Algorithms like Stochastic Gradient Langevin Dynamics (SGLD) extend classical MCMC to the big-data setting by leveraging stochastic gradients. However, the loss landscape of deep neural networks is notoriously complex, characterized by pathological curvature and saddle points \cite{kim2020stochasticgradientlangevindynamics}. Several methods have introduced adaptive step sizes or preconditioning to improve the convergence of SGMCMC on challenging loss landscapes, including geometry-based schemes such as SGRLD and SGRHMC \cite{NIPS2013_309928d4,ma2015completerecipestochasticgradient}, and practical variants like pSGLD \cite{li2015preconditionedstochasticgradientlangevin}. However, as discussed in \cite{ma2015completerecipestochasticgradient,rensmeyer2024convergencelocallyadaptivescalable} and Section~\ref{sec:adaptive_sgmcmc}, these methods are biased unless the dynamics is augmented by a computationally expensive divergence term. Adaptive stepsizes can be viewed as an isotropic but dynamic preconditioning framework \cite{94965ae04e754347bcc2a39a10fc16dd}. Building on the recent formulation of \cite{leimkuhler2025langevinsamplingalgorithminspired}, we revisit adaptive step size methods for SGLD in Bayesian sampling by introducing SA-SGLD.  Importantly this scheme circumvents the computation of the divergence by use of statistical reweighting. We provide theoretical foundations and show using small examples and a Bayesian neural network that this method can improve performance compared to SGLD.

\section{Background}
\subsection{Stochastic Gradient MCMC (SGMCMC)}
SGMCMC is a family of sampling methods that seeks to perform efficient inference with noisy gradients as a result of using mini-batches of data. Suppose we have a dataset $\mathcal{D}=\{d_i\}_{i=1}^N$ and a model parameterized by $\bm{\theta}$. The likelihood is defined by $p(\mathcal{D}|\bm{\theta})$. In a Bayesian setting, we define the prior $p(\bm{\theta})$ and a posterior $p(\bm{\theta}|\mathcal{D}) \propto p(\mathcal{D}|\bm{\theta})p(\bm{\theta})$. This may be re-expressed as a Boltzmann-Gibbs distribution $p(\bm{\theta}|\mathcal{D}) \propto e^{-U(\bm{\theta})}$ where $U(\bm{\theta}) =  \log p(\mathcal{D}|\bm{\theta}) + \log p(\bm{\theta})$
is the potential energy which forms the foundation for SDE-derived sampling procedures. Typically, one may use standard MCMC methods to sample from $p(\bm{\theta}|\mathcal{D})$, provided that we have access to gradients of $\nabla_{\bm{\theta}}U(\bm{\theta})$. However, for a large dataset $\mathcal{D}$, this is costly. Therefore, we approximate $\nabla_{\bm{\theta}}U(\bm{\theta}) \approx \tilde{\nabla_{\bm{\theta}}U(\bm{\theta})}$ with noisy mini-batch gradients. The most commonly used SGMCMC methods is Stochastic Gradient Langevin Dynamics (SGLD) \cite{WellingTeh2011}. In continuous time, SGLD may be written as:
\begin{equation}
    \label{eq:continuous_sgld}
    d\bm{\theta} = -\nabla_{\bm{\theta}}\tilde{U}(\bm{\theta}) dt + \sqrt{2\beta^{-1}}d\bm{W}.
\end{equation}
$\beta^{-1}$ is the temperature to tune the stationary distribution $e^{-\beta U(\bm{\theta})}$. This SDE is discretized using Euler-Maruyama with decreasing stepsizes to remove the bias term asymptotically. However, fixed stepsizes are often used in practice.

% SGHMC includes an auxilliary momentum variable $\bm{p}$, it may be written as:
% \begin{align}
%     \label{eq:continuous_sghmc}
%     d\bm{\theta} &= \bm{p} dt \\
%     d\bm{p} &= (-\nabla_{\bm{\theta}}\tilde{U}(\bm{\theta}) - \gamma \bm{p}) dt + \sqrt{2\beta^{-1}(\gamma-\hat{\gamma})} d\bm{W}.
% \end{align}
% $\gamma$ is a hyperparameter that controls damping while $\hat{\gamma}$ is an estimate of gradient noise. The presence of friction $\gamma \bm{p}$ damps oscillations and speeds up mixing. The stationary distribution is shown \cite{chen2014stochasticgradienthamiltonianmonte} to be $e^{-\beta (U(\bm{\theta}) + \frac{1}{2}\bm{p}^T\bm{p})}$, assuming Gaussian gradient noise.  

%We stress that for SGHMC, one may use more advanced splitting methods to improve convergence and reduce discretization error for SGHMC.
\subsection{Adaptive SGMCMC}
\label{sec:adaptive_sgmcmc}
Adaptive SGMCMC methods improve convergence by adapting to the local geometry of the parameter space. A prominent example is \emph{Stochastic Gradient Riemannian Langevin Dynamics} (SGRLD) \cite{NIPS2013_309928d4,girolami2011Riemann}:
\begin{equation}
    d\bm{\theta} = \big[-\bm{G}(\bm{\theta})^{-1}\nabla_{\bm{\theta}}\tilde{U}(\bm{\theta}) + \bm{\Gamma}(\bm{\theta})\big]\,dt + \sqrt{2\beta^{-1}}\bm{G}(\bm{\theta})^{-\frac{1}{2}}\,d\bm{W},
\end{equation}
where $\bm{G}(\bm{\theta})$ is a positive-definite metric (e.g. Fisher information) and $\Gamma_i(\bm{\theta})=\sum_j \partial_j G_{ij}^{-1}(\bm{\theta})$ is the divergence correction term \cite{girolami2011Riemann,chung2013lectures}. $\bm{G}(\bm{\theta})$ defines a local metric on parameter space, shaping both the drift and diffusion to follow the local curvature of the posterior, improving exploration and convergence. The divergence must be added to the drift to ensure that the sampler preserves the correct stationary distribution. It compensates for the geometric distortion induced by the metric. Neglecting $\bm{\Gamma}$--as is done in practice for many adaptive algorithms such as pSGLD \cite{li2015preconditionedstochasticgradientlangevin}, AdamSGLD \cite{Kim22012022}, and SGRLD in alternative metrics \cite{yu2023scalable}--will break the invariance of the target distribution. These methods replace $\bm{G}(\bm{\theta})$ by an exponentially averaged estimate (e.g.\ RMSProp-like preconditioning) to avoid computing second-order derivatives, which effectively downscales $\bm{\Gamma}$ by $(1-\alpha)$. As shown by \citet{rensmeyer2024convergencelocallyadaptivescalable}, the resulting SDE no longer has the correct stationary density: in one dimension, omitting $\bm{\Gamma}$ yields $\pi(\bm{\theta}) \propto p(\bm{\theta}|\mathcal{D})\,G(\bm{\theta})^{-1},$ and in the general case this results in biased samples unless $G(\bm{\theta})$ is nearly constant over regions of high posterior mass. Thus, the divergence is essential for maintaining ergodicity and convergence to the true posterior \cite{ma2015completerecipestochasticgradient}. Other adaptive methods include ad-hoc stepsize schedules such as Cyclical SGMCMC \cite{zhang2020cyclicalstochasticgradientmcmc} which applies a periodic, cosine‐annealed stepsize schedule to alternate between exploration of new modes and local refinement without explicit curvature estimation. Adaptive Langevin \cite{xiaocheng2015} and SGNHT \cite{ding2014bayesian} introduce an adaptive friction term to adaptively damp the gradient noise. This has also been studied in the context of optimization \cite{fad_2023}\cite{karoni2026adaptivemomentumnonlineardamping}.

\subsection{SAM-ADAMS}
SamAdams \cite{leimkuhler2025langevinsamplingalgorithminspired} is a recently proposed adaptive method for kinetic Langevin dynamics, that adapts the stepsize based on the local geometry of the loss landscape while preserving stability of the dynamics. The key idea is to incorporate a state-dependent time-rescaling. This can be achieved by introducing a new physical time variable $\tau$ such that $\frac{dt}{d\tau} = \psi(\zeta_\tau)$ where $\psi(.)$ defines a Sundman time transformation. The choice of $\zeta$ can, in principle, be arbitrarily chosen. However, a smart choice would be an exponential average of some monitor function $g(\bm{\theta}_\tau)$ that captures local geometry. In \cite{leimkuhler2025langevinsamplingalgorithminspired}, $g(\bm{\theta}_\tau) = \|\nabla_{\theta} U(\bm{\theta})\|^2$ and $\psi(\zeta) = m\frac{\zeta^r + \frac{M}{m}}{\zeta^r + 1}$. This choice of $\psi(.)$ bounds the adaptive stepsize $m\Delta \tau \leq \Delta t \leq M\Delta \tau$. Overall, SamAdams may be written in continuous time as
\begin{align}
    \label{eq:samadams}
    d\bm{\theta} &= \bm{p} dt \\
    d\bm{p} &= (-\nabla U(\bm{\theta}) - \gamma \bm{p}) dt + \sqrt{2\beta^{-1}\gamma }d\bm{W}_t \\
    dt &= \psi(\zeta_{\tau}) d\tau \\
    d\zeta_{\tau} &= -\alpha \zeta_\tau d\tau + g(\bm{\theta}_\tau) d\tau
\end{align}
Unlike fixed or scheduled stepsizes, $\eta(\bm{\theta})$ automatically contracts in regions of high curvature and expands in flatter regions, thus improving both stability and mixing. The invariant measure $e^{-\beta (U(\bm{\theta}) + \frac{1}{2}\bm{p}^T\bm{p})}$ is preserved as time is only rescaled. It is important to note that $\zeta$ here is a scalar and does not adapt each degree of freedom separately.

\section{Methodology}
The SamAdams algorithm in \cite{leimkuhler2025langevinsamplingalgorithminspired} was introduced with exact gradients and in the kinetic Langevin setting. While that article mentions possible applications with noisy gradients the experiments in BNNs used a high accuracy approximate gradient.  The algorithm presented here (SA-SGLD) is specifically viewed as a noisy gradient scheme for large scale
training and is based on the following equations modifying SGLD:
\begin{align*}
    \label{eq:SA-SGLD}
    \zeta_{n+1} &= \rho \zeta_{n} + \frac{(1-\rho)}{\alpha}g(\bm{\theta}_n) \\
    \Delta t_{n+1} &= \psi(\zeta_{n+1}) \Delta \tau \\
    \bm{\theta}_{n+1} &= -\nabla \tilde{U}(\bm{\theta}_{n}) \Delta t_{n+1} + \sqrt{2 \beta^{-1} \Delta t_{n+1}} \varepsilon \\
\end{align*}
where $\rho = \exp(-\alpha \Delta \tau)$, $g(\bm{\theta}_\tau) = \|\nabla_{\theta} \tilde{U}(\bm{\theta})\|^2 + \delta$, $\varepsilon \sim \mathcal{N}(\bm{0}, \bm{I})$. $\delta$ is a smaller regularization term to lower bound $\zeta$ from 0. SA-SGLD may be viewed as a middle ground between two extremes. On the one hand, full Riemannian methods such as SGRLD \cite{NIPS2013_309928d4} incorporate a metric tensor $G(\bm{\theta})$ and a divergence correction term to preserve the stationary distribution, but the correction is prohibitively expensive in high dimensions. On the other hand, heuristic schemes like pSGLD \cite{li2015preconditionedstochasticgradientlangevin} adapt learning rates using optimization-inspired preconditioners (e.g.\ RMSProp), but neglect the correction term, which biases the stationary distribution. SamAdams achieves a compromise: it adapts stepsizes in a principled manner, avoiding the costly divergence computation, and retains the correct stationary distribution. This balance makes it both practical and robust for Bayesian deep learning. The additional computational cost is negligible as one only needs to compute the gradient norm and store a scalar.

\section{Theoretical Results}
The proofs for the following theorems are in Appendix \ref{sec:reweighting_app}
\begin{lemma}[Uniform moment bounds]
\label{lemma:sasgld_moment_bounds_corrected}
Assume the following:
\begin{enumerate}
    \item The monitor uses $g(\theta_n)=\|G_n\|^2 + \delta$, so that
    \begin{equation*}
        \zeta_{n+1} = \rho \zeta_n + \frac{1-\rho}{\alpha}g(\theta_n), 
        \quad \rho = e^{-\alpha h}.
    \end{equation*}
    \item The adaptive time step is $\Delta t_{n+1} = \psi(\zeta_{n+1}) h$,
    where $\psi$ is bounded and globally Lipschitz:
    \begin{align*}
        0 < m &\le \psi(\zeta) \le M < \infty, \\
        |\psi(x)-\psi(y)| &\le L_\psi |x-y|.
    \end{align*}
    \item The potential $U$ is $L$-smooth and dissipative:
    \begin{align*}
        \|\nabla U(\theta) - \nabla U(\theta')\| &\le L\|\theta - \theta'\|,\\
        \langle \theta, \nabla U(\theta) \rangle &\ge a\|\theta\|^2 - b,
    \end{align*}
    for $L > 0$, $a > 0$, $b \ge 0$.
    \item The stochastic gradients satisfy, for some $\sigma < \infty$:
    \begin{align*}
        \mathbb{E}[G_n | \mathcal{F}_n] &= \nabla U(\theta_n), \\
        \mathbb{E}[\|G_n-\nabla U(\theta_n)\|^2|\mathcal F_n]
        &\le \sigma^2(1+\|\theta_n\|^2),
    \end{align*}
    where $\mathcal{F}_n = \sigma(\theta_0,\zeta_0,\varepsilon_1,\ldots,
    \varepsilon_n,G_0,\ldots,G_{n-1})$.
\end{enumerate}

Define the constants:
\begin{align*}
    C_1 &:= 2L^2 + 2\sigma^2, \\
    C_2 &:= 2\|\nabla U(0)\|^2 + 2\sigma^2, \\
    C_3 &:= 2\beta^{-1}d.
\end{align*}
If $h>0$ is small enough that
\begin{equation}
\label{eq:gamma_condition_final}
\gamma(h) := 2 a m h - C_1 M^2 h^2 - 2\sigma M h > 0,
\end{equation}
then the iterates of
\begin{align}
\theta_{n+1}
&= \theta_n - \Delta t_{n+1} G_n 
   + \sqrt{2\beta^{-1}\Delta t_{n+1}}\,\varepsilon_{n+1},
   \label{eq:sasgld_update_final}\\
\varepsilon_{n+1}&\sim \mathcal{N}(0,I_d),
\end{align}
satisfy the uniform moment bound
\begin{equation*}
    \sup_{n\ge 0} \mathbb{E}\|\theta_n\|^{2} < \infty.
\end{equation*}
\end{lemma}

\begin{theorem}[Ergodicity and $O(h)$ bias]
\label{thm:sasgld_bias_fixed_revised}
Retain the assumptions of Theorem~\ref{lemma:sasgld_moment_bounds_corrected}. Further assume:
\begin{enumerate}
    \item For $p > 0$ large, $\sup_{n\ge 0}\mathbb{E}\|\theta_n\|^{p} < \infty$
    \item $\mathbb{E}[\|G_n - \nabla U(\theta_n)\|^4 | \mathcal{F}_n] 
    \le \sigma_4^2(1 + \|\theta_n\|^4).$
    \item $(\theta_n, \zeta_n)$ is ergodic Markov with invariant measure $\widetilde\pi_h$.
    \item $U \in C^4(\mathbb{R}^d)$ with bounded derivatives, $\nabla U$ Lipschitz.
\end{enumerate}

\noindent Let $f:\mathbb{R}^d\to\mathbb{R}$ such that 
\begin{equation*}
\mathcal{L}\phi = f - \pi(f), 
\quad \mathcal{L} = -\nabla U\cdot\nabla + \beta^{-1}\Delta,
\end{equation*}
admits $\phi\in C^4(\mathbb{R}^d)$ with polynomial-growth derivatives:
\begin{equation*}
\sup_{\theta}\frac{\|D^j\phi(\theta)\|}{1+\|\theta\|^{q}} \le A_j, 
\quad j=0,1,2,3,4.
\end{equation*}

\noindent Define the weighted time-average:
\begin{equation*}
\mathcal{A}_n := 
\frac{\sum_{k=1}^n \Delta t_k f(\theta_k)}{\sum_{k=1}^n \Delta t_k}, 
\quad \Delta t_k=\psi(\zeta_k)h.
\end{equation*}
Let the $\psi$-weighted marginal be $\pi_h(f) := \frac{\widetilde\pi_h[\psi(\zeta)f(\theta)]}{\widetilde\pi_h[\psi(\zeta)]}$. Then,
For every $f$:
\begin{equation*}
\mathcal{A}_n \xrightarrow{n\to\infty}{\mathrm{a.s.}} \pi_h(f).
\end{equation*}

\noindent There exists $C>0$ such that
\begin{equation*}
|\pi_h(f)-\pi(f)| \le C h.
\end{equation*}
Moreover, for all $n\ge1$:
\begin{equation*}
\big|\mathbb{E}\mathcal{A}_n - \pi(f)\big|
\le |\pi_h(f)-\pi(f)| + \frac{C'}{n}
= O(h) + O(n^{-1}).
\end{equation*}
\end{theorem}

\section{Experiments}

\subsection{Motivational Examples}
We illustrate the benefit of SA-SGLD by mimicking two common challenges in Neural Network loss landscapes: multi-modal adaptation and high curvature regions, with 2D representative examples below.

\subsubsection{M\"uller--Brown potential.}
The M\"uller--Brown potential \cite{muller1979LocationSaddlePointsa} is a classical benchmark for studying transitions between metastable wells. 
As shown in Figure~\ref{fig:MB}, the SAM-ADAMS-enhanced variant (SA-SGLD) adapts its step size according to local curvature. 
This dynamic adjustment reflects what we would expect in high-dimensional BNN posteriors—where different modes correspond to functionally distinct solutions separated by narrow energy barriers \cite{garipov2018loss}.

\subsubsection{Star potential.}
The second example, the ``star'' potential~\cite{leimkuhler2025langevinsamplingalgorithminspired}, exhibits strong anisotropy and narrow funnel-shaped regions which may be induced in a BNN posterior by incorporating sparse priors. As shown in Figure~\ref{fig:star}, SGLD oversamples the outer ridges and struggles to penetrate the high-curvature funnels, whereas SA-SGLD adapts its step size to enter these regions smoothly, producing a more balanced exploration of the posterior geometry.

\begin{figure}[H]
    \centering
    \includegraphics[width=0.5\linewidth]{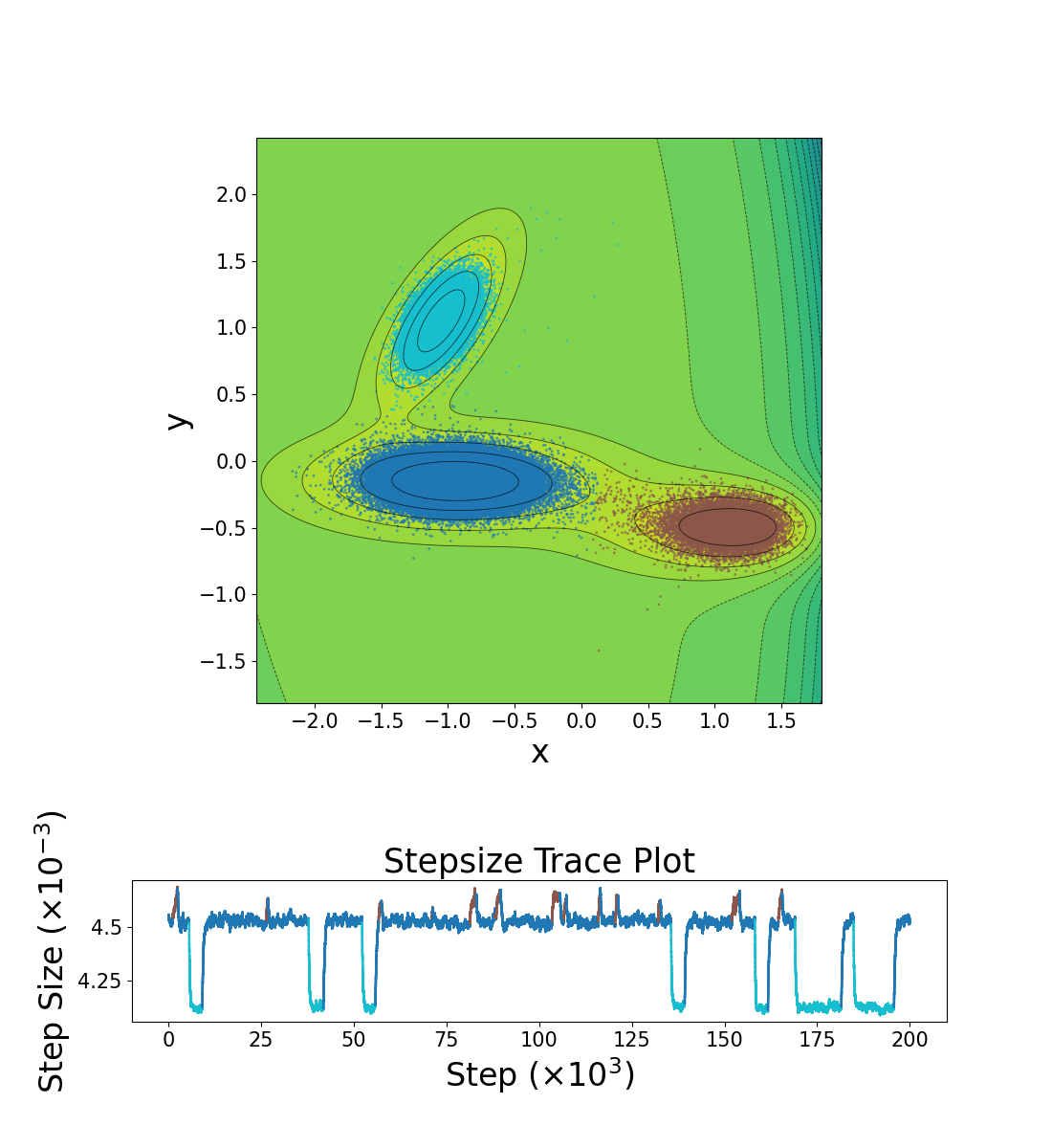}
    \caption{\textbf{M\"uller--Brown potential.}
    SA-SGLD adapts its step size to local curvature, enabling transitions across energy barriers—analogous to escaping local modes in complex BNN posteriors.}
    \label{fig:MB}
\end{figure}

\begin{figure}[H]
    \centering
    \includegraphics[width=0.5\linewidth]{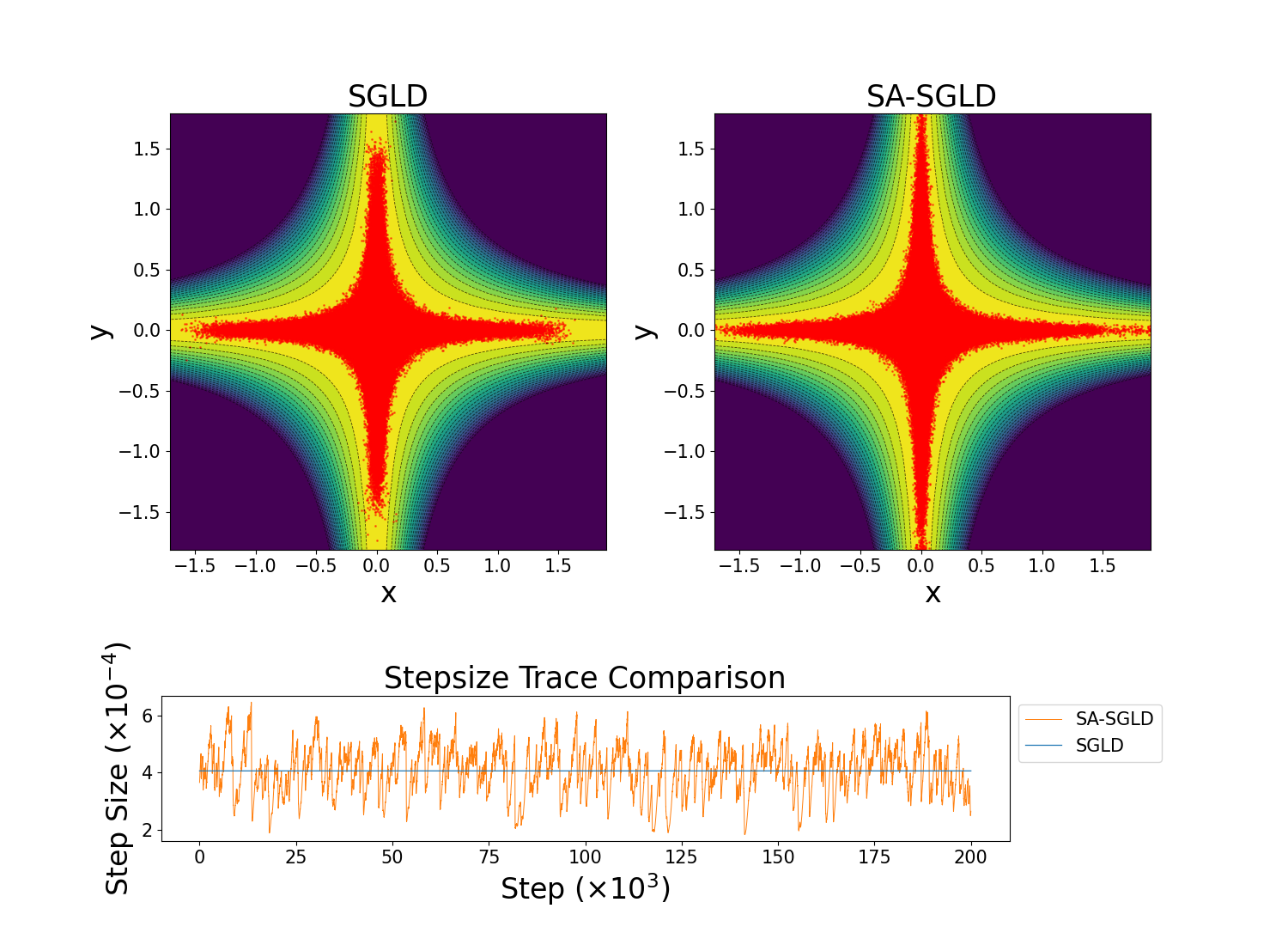}
    \caption{\textbf{Star potential.}
    SA-SGLD dynamically scales its step size, entering narrow high-curvature funnels that SGLD fails to explore.}
    \label{fig:star}
\end{figure}

\noindent Further implementation details, parameter values, and potential definitions are provided in Appendix A.

\subsection{Bayesian Neural Networks}
\cite{leimkuhler2025langevinsamplingalgorithminspired} demonstrated the effect of SAM-ADAMS on kinetic Langevin dynamics for neural network training at a very large batch size (10000/50000), effectively eliminating the role of stochastic gradient noise. Moreover, they simply assessed the dynamics of the collected samples rather than making evaluations using the collected ensemble. As a result, we conduct experiments in a proper BNN setup with appropriate metrics (e.g ECE/NLL/Acc) after ensembling, Bayesian prior setups, and a small batch size. We use the MNIST dataset a 3 layer fully connected neural network (FCNN) at a batch size of 100. We aim to estimate the posterior predictive $p(\bm{y}|\bm{x}) = \int p(\bm{y}|\bm{x}, \bm{\theta})p(\bm{\theta}|\mathcal{D}) d\bm{\theta}$ with a Monte Carlo estimate $\frac{1}{n}\Sigma_{i=1}^{n} p(\bm{y}|\bm{x},\bm{\theta})$. According to \cite{fortuin2022bayesian}, for FCNNs, uncorrelated heavy-tailed priors (e.g Student-t, Horseshoe) perform well. Such heavy-tailed priors induce sparsity in the parameter-space of the neural network and create funnel-like geometries in the posterior that SA-SGLD may be well-suited to adapting to. We use Cross Entropy Loss, and compare between a Horseshoe and Gaussian prior on the parameters. We run 200 epochs and collect samples with a thinning interval of 100 batches. This results in a total of 1000 models. 

\noindent We evaluate performance with: negative log-likelihood (NLL), expected calibration error (ECE), and test accuracy. We selected standard values for $r=0.25, s=2$ and picked $m=0.5, M=2, \alpha=1000, d\tau=0.2$ for Table \ref{tab:sgld_vs_sa} and Fig \ref{fig:exp_1_fig}. The first 100 epochs of the samples are discarded as burnin.

\begin{table}[H]
\centering
\small
\setlength{\tabcolsep}{3pt}
\begin{tabular}{@{}lccc@{}}
\hline
\textbf{Method} & \textbf{NLL $\downarrow$} & \textbf{Acc (\%) $\uparrow$} & \textbf{ECE (\%) $\downarrow$} \\
\hline
\multicolumn{4}{c}{\textbf{Gaussian Prior}} \\
\hline
SGLD & $\bm{0.192} \pm 0.005$ & $95.23 \pm 0.056$ & $\bm{5.72} \pm 0.27$ \\
SA-SGLD          & $0.193 \pm 0.004$ & $\bm{95.25} \pm 0.034$ & $5.76 \pm 0.22$ \\
\hline
\multicolumn{4}{c}{\textbf{Horseshoe Prior}} \\
\hline
SGLD & $0.086 \pm 0.004$ & $98.03 \pm 0.044$ & $3.64 \pm 0.11$ \\
SA-SGLD          & $\bm{0.080} \pm 0.003$ & $\bm{98.12} \pm 0.031$ & $\bm{3.49} \pm 0.9$ \\
\hline
\end{tabular}
\caption{Comparison of predictive performance between SGLD and SA-SGLD for Gaussian and Horseshoe priors.}
\label{tab:sgld_vs_sa}
\end{table}

\noindent We find that for Gaussian priors, SA-SGLD does not offer much benefit. However, for a Horseshoe prior, there is a substantial benefit as indicated by improved NLL, Test Accuracy, and ECE. We suspect this is due to the high curvature induced by the Horseshoe prior. Furthermore, we observe in Fig \ref{fig:exp_1_fig} that as more ensembles are collected past the burnin period, the log-probability of SA-SGLD exceeds that of SGLD.
\begin{figure}[H]
  \centering
  \begin{minipage}[b]{0.48\linewidth}
    \centering
    \includegraphics[width=\linewidth]{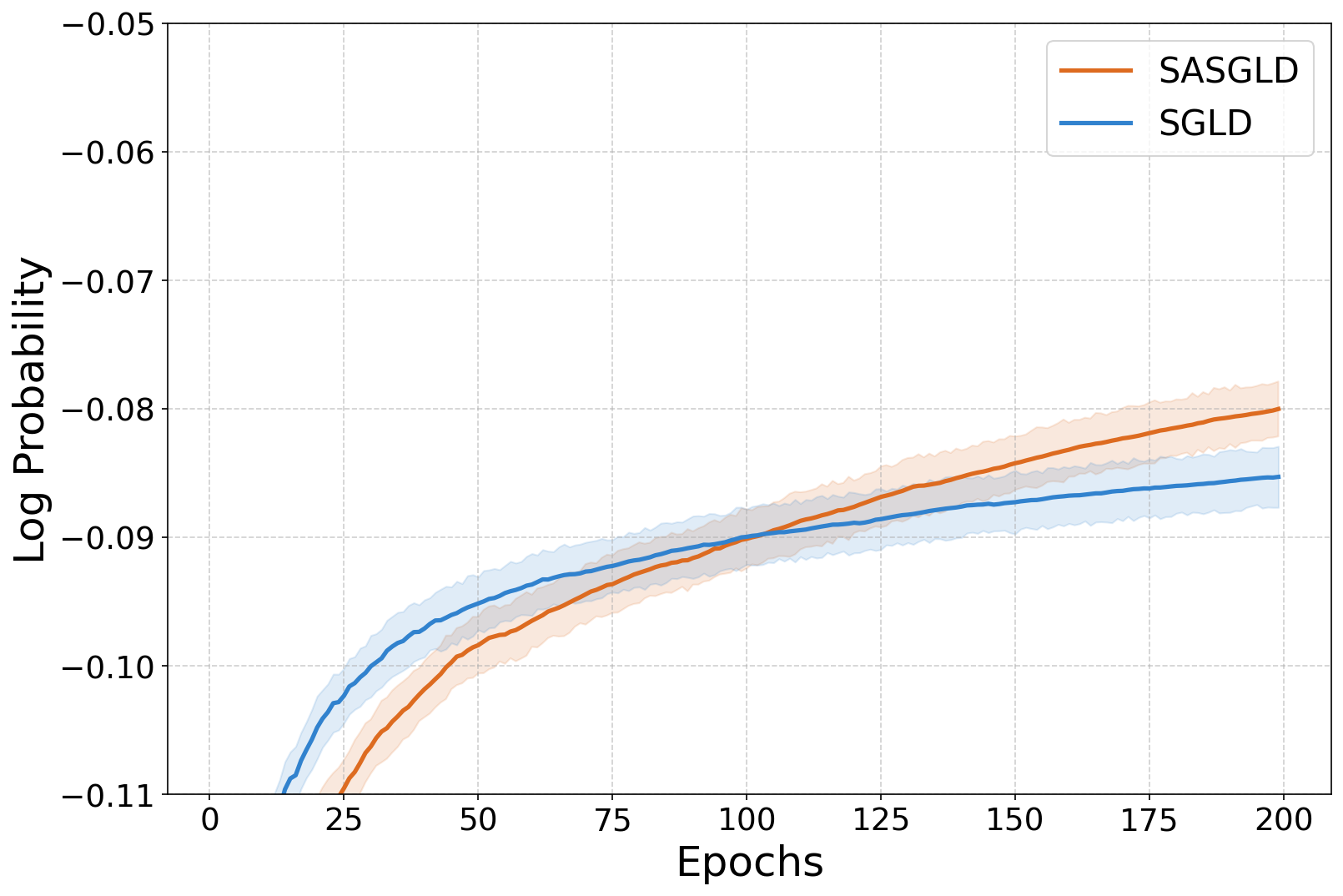}
    \smallskip\\
    \textbf{(a)} Log Probability Curves
  \end{minipage}
  \hfill
  \begin{minipage}[b]{0.48\linewidth}
    \centering
    \includegraphics[width=\linewidth]{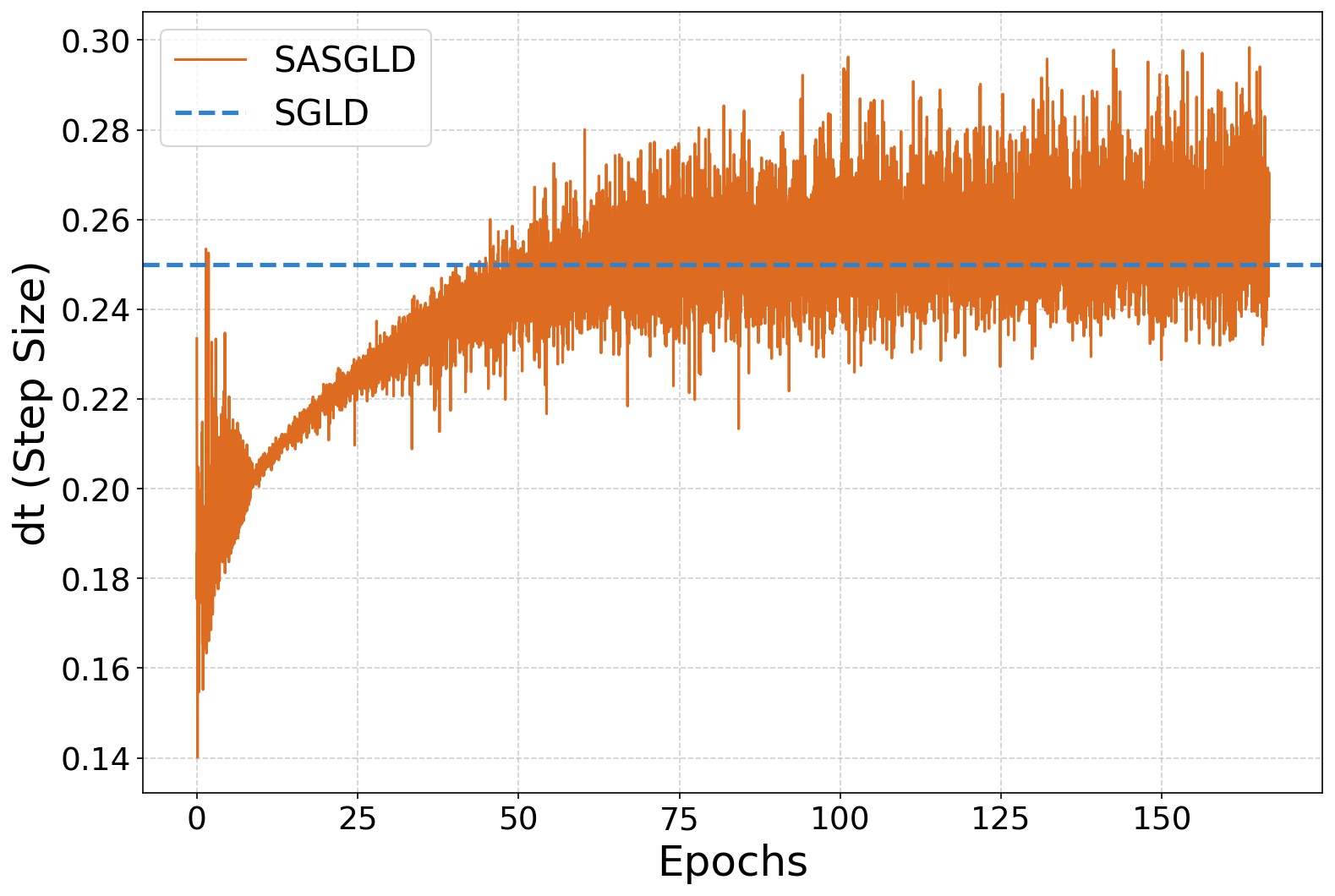}
    \smallskip\\
    \textbf{(b)} Stepsize Trace
  \end{minipage}
  \caption{SGLD vs SA-SGLD on sampling BNN with Horseshoe prior on MNIST data. Log Probability shown is computed with the entire ensemble until that epoch.}
  \label{fig:exp_1_fig}
\end{figure}

\noindent Fig \ref{fig:exp_2_fig} indicates that for $h \in [0.45, 0.55]$, SGLD's log probability diverges whereas SA-SGLD manages to retain performance even at a mean stepsize of 0.5. We only increased $d\tau$ to 0.35 and $r$ to 0.5. SGLD ($h=0.55$) diverged early and is not shown in the figure.

\begin{figure}[H]
  \centering
  \begin{minipage}[b]{0.48\linewidth}
    \centering
    \includegraphics[width=\linewidth]{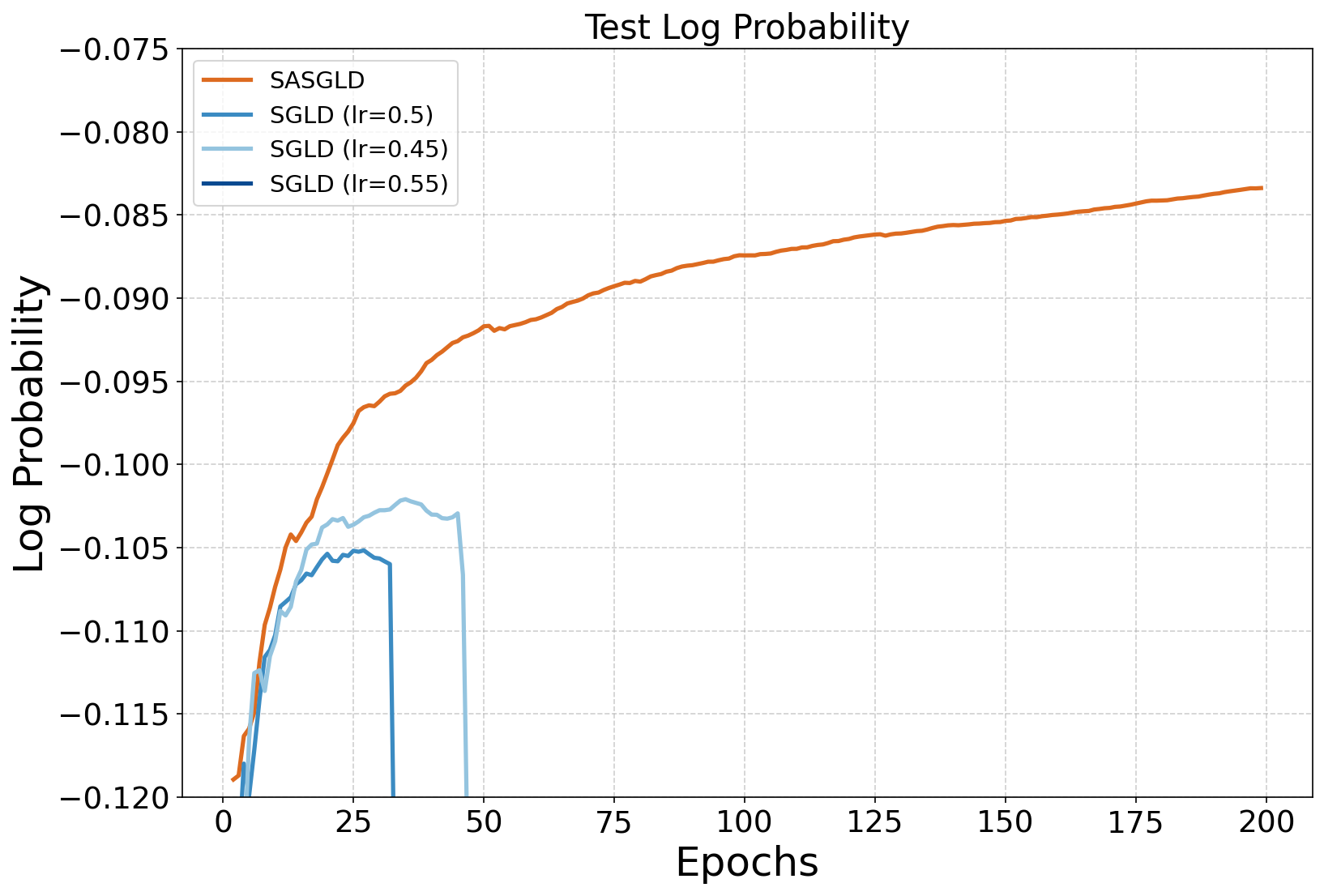}
    \smallskip\\
    \textbf{(a)} Log Probability Curves
  \end{minipage}
  \hfill
  \begin{minipage}[b]{0.48\linewidth}
    \centering
    \includegraphics[width=\linewidth]{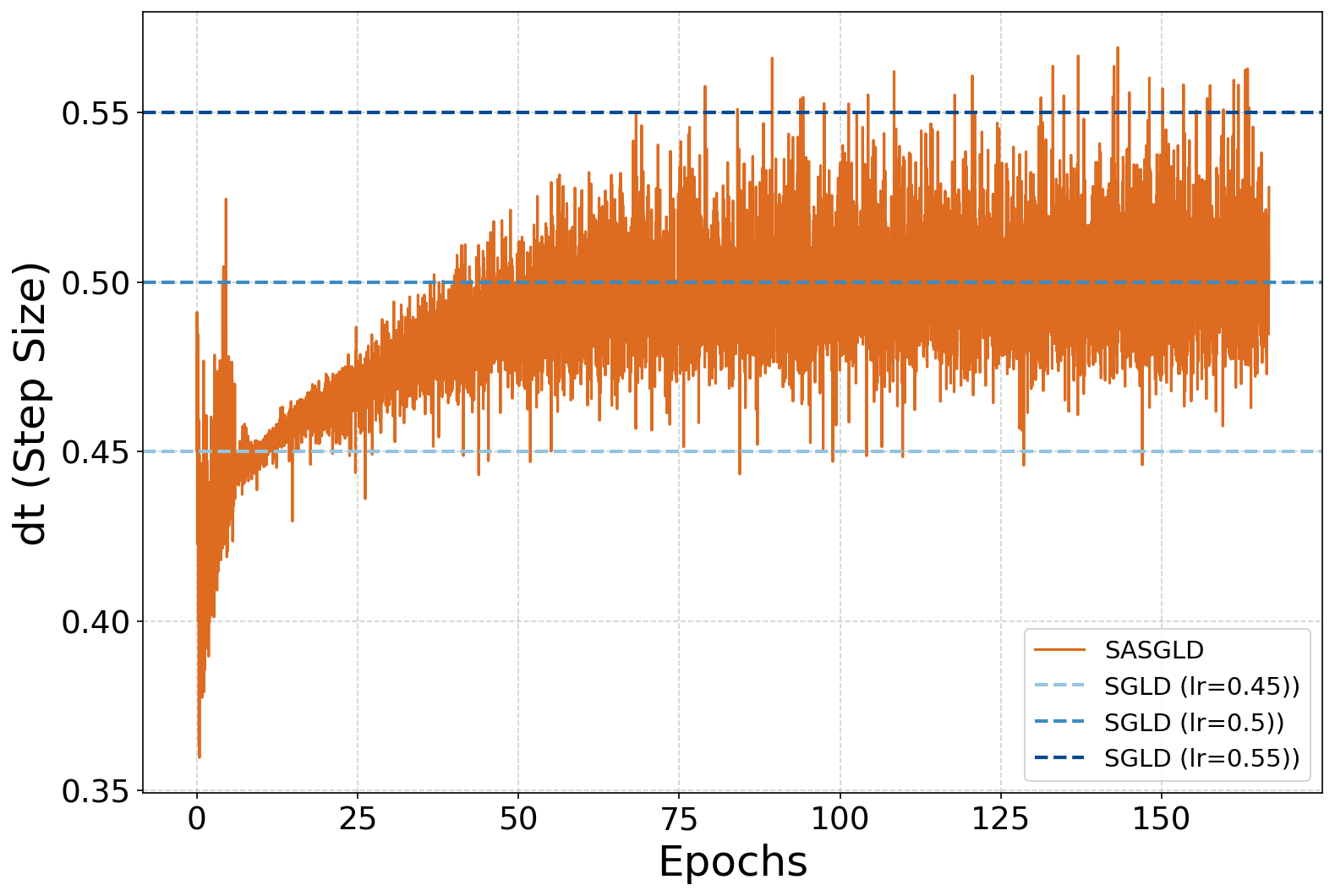}
    \smallskip\\
    \textbf{(b)} Stepsize Trace
  \end{minipage}
  \caption{SA-SGLD's robustness with large stepsizes.}
  \label{fig:exp_2_fig}
\end{figure}

\section{Conclusion}
We introduce a principled adaptive framework for stochastic gradient MCMC based on the SAM-ADAMS time-rescaling mechanism. By modulating stepsizes according to local gradient geometry, our approach preserves the correct invariant distribution while avoiding the expensive divergence corrections.

\section{Limitations and Future Work}
Extending SAM-ADAMS to a broader class of monitor functions or matrix-valued adaptations is a promising avenue, assuming complexity concerns can be addressed. Future work may include extending the method to other SGMCMC variants (e.g. SGHMC, SGNHT), exploring multidimensional adaptation, studying it as an optimizer in the zero-temperature limit, and applying it to wider range of probabilistic modeling tasks.

% % \subsubsection{Acknowledgments.}
% % This research was supported by the MAC-MIGS Centre for Doctoral Training.
\bibliographystyle{plainnat}
\bibliography{aaai25}
\appendix
\section{Experimental Details}\label{app:toy-details}
\paragraph{M\"uller--Brown potential.}
The M\"uller--Brown potential~\cite{muller1979LocationSaddlePointsa} is defined as
\begin{align*}
    U(x,y) &= \frac{1}{20}\sum_{i=1}^{4} C_i \exp\{l_i\}, \\
    l_i &= a_i(x-u_i)^2 + b_i(x-u_i)(y-v_i) + c_i(y-v_i)^2,
\end{align*}
where \(C_i, a_i, b_i, c_i, u_i, v_i\), for \(i=1,\ldots,4\), are the potential parameters. 
We used
\begin{align*}
    C_i &= [-267.0, -285.0, -275.0, 2.5],\\
    a_i &= [-0.9,   -0.9,   -9.5,   0.6],\\
    b_i &= [0.0,    0.0,    10.0,  0.1],\\
    c_i &= [-9.0, -9.0, -5.5, 0.1],\\
    u_i &= [1.35,   -0.95,  -1.05, -1.0],\\
    v_i &= [-0.5,   -0.15,   1.05,   0.9].
\end{align*}
The adaptive parameters for SA-SGLD were \(\alpha=0.07\), \(r=0.25\), and \(s=2\).
\paragraph{Star potential.}
The ``star'' potential~\cite{leimkuhler2025langevinsamplingalgorithminspired} is defined as
\[
U(x,y) = x^2 + 1000x^2y^2 + y^2,
\]
with adaptive parameters \(\alpha=0.5\), \(r=0.5\), and \(s=2\).

\noindent Each simulation used $10^6$ samples with a fixed base step size for SGLD and a dynamically evolving step size for SA-SGLD. The corresponding trace plots visualize how the adaptive mechanism responds to local curvature across modes.

\paragraph{Bayesian Neural Network.}
All experiments were run on a NVIDIA V100 40GB GPU. Following the setup from \cite{yu2023scalable} and \cite{li2015preconditionedstochasticgradientlangevin}, we use the MNIST dataset a 3 layer fully connected neural network (FCNN) 784-N-N-10, where $N=1200$.  We do not use data augmentation to avoid the cold posterior effect to sample at temperature $T=1$.  All results reported as mean $\pm$ 95\% confidence interval over 5 runs. We note that the ECE may be ambiguous to interpret as it is not a proper scoring rule as a random classifier could have an ECE of 0. NLL on the other hand is a proper scoring rule.

\section{Theoretical Results}
\label{sec:reweighting_app}
Our analysis builds upon foundational work in stochastic approximation and Markov chain theory. The moment bounds in Theorem 1 employ Lyapunov drift techniques inspired by \cite{Raginsky2017}'s non-asymptotic analysis of non-convex learning, while the ergodicity framework follows the Markov chain stability theory of \cite{Meyn2009}. For the bias analysis in Theorem 2, we adapt the Poisson equation approach developed by \cite{Mattingly2010} for numerical time-averaging, combined with the non-asymptotic analysis of stochastic gradient Langevin dynamics by \cite{Vollmer2016}. The local error analysis draws from \cite{Sato2014}'s Fokker-Planck approximation framework, and the handling of adaptive step-sizes utilizes the stochastic approximation theory of \cite{Kushner2003}. 

\noindent We note that our proof does not yield optimal convergence rates and that inequalities can be tightened further. We used simplifying assumptions for theoretical convenience such as $(\theta_n, \zeta_n)$ being ergodic Markov with unique invariant measure. This is reasonable to assume because $\zeta > 0$ and $m \leq \psi(.) \leq M$, however it must still be shown rigorously. We also assumed $\sup_{n\ge 0} \mathbb{E}\|\theta_n\|^{p} < \infty$ for arbitrary $p$ while only proving it for $p=2$ in Theorem 1. This must also be proven rigorously and we leave it to future work. The regularization parameter $\delta > 0$ ensures $\zeta_n$ is bounded away from zero, allowing $\psi$ to be globally Lipschitz on that domain. 
All convergence constants depend on $\delta$ but remain finite for any fixed $\delta > 0$.
%%%%%%%%%%%%%%%%%%%%%%%%%%%%%%%%%%%%%%%%%%%%%%%%%%%%%%%%%%%%%%%%%%%%%%%%%%
% THEOREM 1: Uniform moment bounds for SA--SGLD (CORRECTED)
%%%%%%%%%%%%%%%%%%%%%%%%%%%%%%%%%%%%%%%%%%%%%%%%%%%%%%%%%%%%%%%%%%%%%%%%%%
\setcounter{lemma}{0} % one less than 1
\begin{lemma}[Uniform moment bounds]
\label{lemma:sasgld_moment_bounds_corrected}
Assume the following:
\begin{enumerate}
    \item The Sundman step is fixed: $\Delta\tau_n \equiv h>0$.
    \item The monitor uses $g(\theta_n)=\|G_n\|^2 + \delta$, so that
    \begin{equation*}
        \zeta_{n+1} = \rho \zeta_n + \frac{1-\rho}{\alpha}\|G_n\|^2, 
        \quad \rho = e^{-\alpha h}.
    \end{equation*}
    \item The adaptive time step is $\Delta t_{n+1} = \psi(\zeta_{n+1}) h$,
    where $\psi$ is bounded and globally Lipschitz:
    \begin{align*}
        0 < m &\le \psi(\zeta) \le M < \infty, \\
        |\psi(x)-\psi(y)| &\le L_\psi |x-y|.
    \end{align*} for $\zeta \geq \frac{(1-\rho)\delta}{\alpha}$.
    \item The potential $U$ is $L$-smooth and dissipative:
    \begin{align*}
        \|\nabla U(\theta) - \nabla U(\theta')\| &\le L\|\theta - \theta'\|,\\
        \langle \theta, \nabla U(\theta) \rangle &\ge a\|\theta\|^2 - b,
    \end{align*}
    for $L > 0$, $a > 0$, $b \ge 0$.
    \item The stochastic gradients satisfy, for some $\sigma < \infty$:
    \begin{equation*}
        \mathbb{E}[\|G_n-\nabla U(\theta_n)\|^2|\mathcal F_n]
        \le \sigma^2(1+\|\theta_n\|^2),
    \end{equation*}
    where $\mathcal{F}_n = \sigma(\theta_0,\zeta_0,\varepsilon_1,\ldots,
    \varepsilon_n,G_0,\ldots,G_{n-1})$.
\end{enumerate}

Define the constants:
\begin{align*}
    C_1 &:= 2L^2 + 2\sigma^2, \\
    C_2 &:= 2\|\nabla U(0)\|^2 + 2\sigma^2, \\
    C_3 &:= 2\beta^{-1}d.
\end{align*}
If $h>0$ is small enough that
\begin{equation}
\label{eq:gamma_condition_final}
\gamma(h) := 2 a m h - C_1 M^2 h^2 - 2\sigma M h > 0,
\end{equation}
then the iterates of
\begin{align}
\theta_{n+1}
&= \theta_n - \Delta t_{n+1} G_n 
   + \sqrt{2\beta^{-1}\Delta t_{n+1}}\,\varepsilon_{n+1},
   \label{eq:sasgld_update_final}\\
\varepsilon_{n+1}&\sim \mathcal{N}(0,I_d),
\end{align}
satisfy the uniform moment bound
\begin{equation*}
    \sup_{n\ge 0} \mathbb{E}\|\theta_n\|^{2} < \infty.
\end{equation*}
\end{lemma}
\begin{proof}
\noindent Since $\Delta t_{n+1}=\psi(\zeta_{n+1})h$ depends on $G_n$, it is not $\mathcal F_n$--measurable. 
All conditional expectations keep $\Delta t_{n+1}$ inside the expectation.

\noindent Expanding \eqref{eq:sasgld_update_final}:
\begin{align*}
\|\theta_{n+1}\|^2
&= \|\theta_n\|^2 
   - 2\Delta t_{n+1}\langle \theta_n, G_n\rangle \\
   &\quad + \Delta t_{n+1}^2\|G_n\|^2
   + 2\beta^{-1}\Delta t_{n+1}\|\varepsilon_{n+1}\|^2 + R_n,
\end{align*}
where $R_n$ contains cross terms. Taking conditional expectations using 
$\mathbb{E}[\|\varepsilon_{n+1}\|^2|\mathcal F_n]=d$ and $\mathbb{E}[\varepsilon_{n+1}|\mathcal F_n]=0$:
\begin{equation}
\label{eq:cond_expand_final}
\begin{split}
\mathbb{E}[\|\theta_{n+1}\|^2|\mathcal F_n]
&= \|\theta_n\|^2
-2\,\mathbb{E}[\Delta t_{n+1}\langle\theta_n,G_n\rangle|\mathcal F_n] \\
&\quad +\mathbb{E}[\Delta t_{n+1}^2\|G_n\|^2|\mathcal F_n]
+2\beta^{-1}d\,\mathbb{E}[\Delta t_{n+1}|\mathcal F_n].
\end{split}
\end{equation}

\noindent Decompose:
\begin{align*}
&\mathbb{E}[\Delta t_{n+1}\langle\theta_n,G_n\rangle|\mathcal F_n] \\
&= \mathbb{E}[\Delta t_{n+1}\langle\theta_n,\nabla U(\theta_n)\rangle|\mathcal F_n] \\
&\quad + \mathbb{E}[\Delta t_{n+1}\langle\theta_n,G_n-\nabla U(\theta_n)\rangle|\mathcal F_n].
\end{align*}
By dissipativity and $m h \le \Delta t_{n+1}\le M h$:
\begin{equation*}
-2\,\mathbb{E}[\Delta t_{n+1}\langle\theta_n,\nabla U(\theta_n)\rangle|\mathcal F_n]
\le -2a m h\,\|\theta_n\|^2 + 2b M h.
\end{equation*}
For the stochastic term, Cauchy--Schwarz gives:
\begin{align*}
&\big|\mathbb{E}[\Delta t_{n+1}\langle\theta_n,G_n-\nabla U(\theta_n)\rangle|\mathcal F_n]\big| \\
&\le M h\,\|\theta_n\|\,\mathbb{E}[\|G_n-\nabla U(\theta_n)\||\mathcal F_n] \\
&\le M h\,\|\theta_n\|\sqrt{\mathbb{E}[\|G_n-\nabla U(\theta_n)\|^2|\mathcal F_n]} \\
&\le M h\sigma\,\|\theta_n\|\sqrt{1+\|\theta_n\|^2}.
\end{align*}
Using Young's inequality $xy \le \frac{x^2}{2} + \frac{y^2}{2}$ with $x = \|\theta_n\|$:
\begin{equation*}
\|\theta_n\|\sqrt{1+\|\theta_n\|^2} \le \|\theta_n\|^2 + \frac{1}{2}.
\end{equation*}
Thus:
\begin{equation}
\label{eq:drift_bound_final}
\begin{split}
&-2\,\mathbb{E}[\Delta t_{n+1}\langle\theta_n,G_n\rangle|\mathcal F_n] \\
&\le -(2a m h - 2\sigma M h)\|\theta_n\|^2 + (2b M h + \sigma M h).
\end{split}
\end{equation}

\noindent From $L$-smoothness and the noise assumption:
\begin{align*}
\mathbb{E}[\|G_n\|^2|\mathcal F_n] 
&\le 2\mathbb{E}[\|\nabla U(\theta_n)\|^2|\mathcal F_n] \\
&\quad + 2\mathbb{E}[\|G_n-\nabla U(\theta_n)\|^2|\mathcal F_n] \\
&\le 2L^2\|\theta_n\|^2 + 2\|\nabla U(0)\|^2 + 2\sigma^2(1+\|\theta_n\|^2) \\
&= C_1\|\theta_n\|^2 + C_2.
\end{align*}
Since $\Delta t_{n+1}^2 \le M^2 h^2$:
\begin{align*}
\mathbb{E}[\Delta t_{n+1}^2\|G_n\|^2|\mathcal F_n]
&\le M^2 h^2(C_1\|\theta_n\|^2 + C_2), \\
2\beta^{-1}d\,\mathbb{E}[\Delta t_{n+1}|\mathcal F_n]
&\le C_3 M h.
\end{align*}

Substituting into \eqref{eq:cond_expand_final}:
\begin{align*}
\mathbb{E}[\|\theta_{n+1}\|^2|\mathcal F_n]
&\le \|\theta_n\|^2 - (2a m h - 2\sigma M h)\|\theta_n\|^2 + C_3 M h \\
&\quad + (2b M h + \sigma M h) 
   + M^2 h^2(C_1\|\theta_n\|^2 + C_2) \\
&= (1 - 2a m h + C_1 M^2 h^2 + 2\sigma M h)\|\theta_n\|^2 \\
&\quad + (C_2 M^2 h^2 + (2b + C_3 + \sigma)M h).
\end{align*}
Let $\kappa(h) := C_2 M^2 h^2 + (2b + C_3 + \sigma)M h$.
If $h>0$ is small enough that $\gamma(h)>0$, taking expectations:
\begin{equation*}
x_{n+1} \le (1-\gamma(h))x_n + \kappa(h), \quad x_n:=\mathbb{E}\|\theta_n\|^2.
\end{equation*}
Since $0<1-\gamma(h)<1$:
\begin{equation*}
\sup_{n\ge0}x_n 
\le \max\{x_0, \kappa(h)/\gamma(h)\} < \infty.
\end{equation*}
\end{proof}

%%%%%%%%%%%%%%%%%%%%%%%%%%%%%%%%%%%%%%%%%%%%%%%%%%%%%%%%%%%%%%%%%%%%%%%%%%
% THEOREM 2: Ergodicity and O(h) bias for fixed-step SA--SGLD (REVISED)
%%%%%%%%%%%%%%%%%%%%%%%%%%%%%%%%%%%%%%%%%%%%%%%%%%%%%%%%%%%%%%%%%%%%%%%%%%
\setcounter{theorem}{0} % one less than 1
\begin{theorem}[Ergodicity and $O(h)$ bias]
\label{thm:sasgld_bias_fixed_revised}
Retain Lemma~\ref{lemma:sasgld_moment_bounds_corrected} assumptions. Further assume:
\begin{enumerate}
    \item $\mathbb{E}[G_n | \mathcal{F}_n] = \nabla U(\theta_n)$.
    \item For $p > 0$ large: $\sup_{n\ge 0} \mathbb{E}\|\theta_n\|^{p} < \infty,$
    \item $\mathbb{E}[\|G_n - \nabla U(\theta_n)\|^4 | \mathcal{F}_n] 
    \le \sigma_4^2(1 + \|\theta_n\|^4).$
    \item $(\theta_n, \zeta_n)$ is ergodic Markov with invariant measure $\widetilde\pi_h$.
    \item $U \in C^4(\mathbb{R}^d)$ with bounded derivatives, $\nabla U$ Lipschitz.
\end{enumerate}

Let $f:\mathbb{R}^d\to\mathbb{R}$ such that 
\begin{equation*}
\mathcal{L}\phi = f - \pi(f), 
\quad \mathcal{L} = -\nabla U\cdot\nabla + \beta^{-1}\Delta,
\end{equation*}
admits $\phi\in C^4(\mathbb{R}^d)$ with polynomial-growth derivatives:
\begin{equation*}
\sup_{\theta}\frac{\|D^j\phi(\theta)\|}{1+\|\theta\|^{q}} \le A_j, 
\quad j=0,1,2,3,4.
\end{equation*}

Define the weighted time-average:
\begin{equation*}
\mathcal{A}_n := 
\frac{\sum_{k=1}^n \Delta t_k f(\theta_k)}{\sum_{k=1}^n \Delta t_k}, 
\quad \Delta t_k=\psi(\zeta_k)h.
\end{equation*}
Let $\pi_h(f) := \frac{\widetilde\pi_h[\psi(\zeta)f(\theta)]}{\widetilde\pi_h[\psi(\zeta)]}$ be the $\psi$-weighted marginal. Then,
For every $f$:
\begin{equation*}
\mathcal{A}_n \xrightarrow{n\to\infty}{\mathrm{a.s.}} \pi_h(f).
\end{equation*}

\noindent There exists $C>0$ such that
\begin{equation*}
|\pi_h(f)-\pi(f)| \le C h.
\end{equation*}
Moreover, for all $n\ge1$:
\begin{equation*}
\big|\mathbb{E}\mathcal{A}_n - \pi(f)\big|
\le |\pi_h(f)-\pi(f)| + \frac{C'}{n}
= O(h) + O(n^{-1}).
\end{equation*}
\end{theorem}

\begin{proof}

By strengthened moments, $\sup_n\mathbb{E}\|\theta_n\|^{p}<\infty$. Since $(\theta_n,\zeta_n)$ is ergodic Markov with invariant measure $\widetilde\pi_h$ and $f$ has polynomial growth, $(\theta,\zeta) \mapsto \psi(\zeta)f(\theta)$ is $\widetilde\pi_h$-integrable. 

\noindent By the ergodic theorem for weighted averages and $T_n = \sum_{k=1}^n \Delta t_k \ge mhn \to \infty$ a.s.:
\begin{equation*}
\mathcal{A}_n = \frac{\sum_{k=1}^n \Delta t_k f(\theta_k)}{\sum_{k=1}^n \Delta t_k} 
\xrightarrow{n\to\infty}{\mathrm{a.s.}} \pi_h(f).
\end{equation*}

\noindent Let $\phi$ solve $\mathcal{L}\phi = f - \pi(f)$. Fourth-order Taylor:
\begin{align*}
\phi(\theta_{n+1}) =& \phi(\theta_n) + D\phi(\theta_n)\cdot(\theta_{n+1}-\theta_n) \\
&+ \tfrac12(\theta_{n+1}-\theta_n)^\top D^2\phi(\theta_n)(\theta_{n+1}-\theta_n) \\
&+ \tfrac16 D^3\phi(\theta_n)[\theta_{n+1}-\theta_n]^{\otimes 3} \\
&+ \tfrac{1}{24} D^4\phi(\tilde\theta_n)[\theta_{n+1}-\theta_n]^{\otimes 4}.
\end{align*}

\noindent Substituting $\theta_{n+1}-\theta_n = -\Delta t_{n+1}G_n + 
\sqrt{2\beta^{-1}\Delta t_{n+1}}\,\varepsilon_{n+1}$ and taking 
$\mathbb{E}[\cdot|\mathcal{F}_n]$:
\begin{align*}
\mathbb{E}[\phi(\theta_{n+1})|\mathcal{F}_n] 
=& \phi(\theta_n) 
- \mathbb{E}[\Delta t_{n+1} D\phi(\theta_n)\cdot G_n|\mathcal{F}_n] \\
&+ \beta^{-1}\mathbb{E}[\Delta t_{n+1} \Delta\phi(\theta_n)|\mathcal{F}_n] \\
&+ \tfrac12\mathbb{E}[\Delta t_{n+1}^2 (G_n^\top D^2\phi(\theta_n)G_n)|\mathcal{F}_n] \\
&+ \mathbb{E}[R_n|\mathcal{F}_n],
\end{align*}
where terms with odd powers of $\varepsilon_{n+1}$ vanish.

\noindent By properties of Gaussian noise, the terms 
$\mathbb{E}[\Delta t_{n+1}^{3/2}\langle D^2\phi(\theta_n)G_n, \varepsilon_{n+1}\rangle|\mathcal{F}_n]$ 
and related cross terms vanish.

\noindent Decompose:
\begin{equation*}
\mathbb{E}[\Delta t_{n+1} D\phi(\theta_n)\cdot G_n|\mathcal{F}_n]
= \mathbb{E}[\Delta t_{n+1}|\mathcal{F}_n] D\phi(\theta_n)\cdot\nabla U(\theta_n)
+ B_n,
\end{equation*}
where the bias term $B_n$ arises from correlation between $\Delta t_{n+1}$ 
and $G_n - \nabla U(\theta_n)$. By Lipschitz property of $\psi$ and 
moment bounds:
\begin{equation*}
|B_n| \le C_2 h^2 (1 + \|\theta_n\|^{q+3}).
\end{equation*}

\noindent For third and fourth order terms in the Taylor expansion:
\begin{equation*}
|\mathbb{E}[R_n|\mathcal{F}_n]| \le C_3 h^2 (1 + \|\theta_n\|^{q+4}),
\end{equation*}
using $\|\theta_{n+1}-\theta_n\| \le Mh\|G_n\| + \sqrt{2\beta^{-1}Mh}\|\varepsilon_{n+1}\|$ 
and polynomial growth of derivatives.

\noindent Rearranging the expansion, in stationarity:
\begin{equation*}
\mathbb{E}_{\widetilde\pi_h}[\mathbb{E}[\Delta t_{n+1}|\mathcal{F}_n](f(\theta_n) - \pi(f))]
= \mathbb{E}_{\widetilde\pi_h}[B_n] + O(h^2 \mathbb{E}[\|\theta_n\|^{q+4}]).
\end{equation*}
Since 
\begin{equation*}
\mathbb{E}[\Delta t_{n+1}|\mathcal{F}_n](f-\pi(f))
= h\psi(\zeta_n)(f-\pi(f)) + O(h^2),
\end{equation*}
we obtain:
\begin{equation*}
h \cdot \mathbb{E}_{\widetilde\pi_h}[\psi(\zeta_n)(f - \pi(f))] = O(h^2),
\end{equation*}
thus
\begin{equation*}
|\pi_h(f) - \pi(f)| = O(h).
\end{equation*}

\noindent Summing from $k=0$ to $n-1$:
\begin{equation*}
\sum_{k=0}^{n-1} \mathbb{E}[\Delta t_{k+1}|\mathcal{F}_k](f(\theta_k) - \pi(f))
= \phi(\theta_0) - \phi(\theta_n) + \sum_{k=0}^{n-1} E_k,
\end{equation*}
where $|\mathbb{E}[E_k]| \le C_5 h^2$. Dividing by $T_n \ge mhn$:
\begin{equation*}
|\mathbb{E}\mathcal{A}_n - \pi(f)| \le Ch + \frac{C'}{n}.
\end{equation*}
\end{proof}

\end{document}